\documentclass[letterpaper, 10 pt, conference]{ieeeconf}
\IEEEoverridecommandlockouts                                                 
\usepackage[utf8]{inputenc}
\usepackage{amsmath}

\usepackage{amsfonts}
\usepackage{amsthm}
\usepackage{amssymb}
\usepackage{graphicx}
\usepackage{cite}

\usepackage[caption=false]{subfig}

\title{\LARGE \bf A Null-space based Approach for a Safe and Effective Human-Robot Collaboration}
\author{Federico Benzi and Cristian Secchi
\thanks{F. Benzi and C. Secchi are
    with the Department of Sciences and Methods of Engineering,
    University of Modena and Reggio Emilia, Italy  {\tt\small
      \{federico.benzi,cristian.secchi\}@unimore.it}}}

\date{February 2021}

\usepackage{xcolor}

\newcommand{\fb}[1] {{\color{black}#1}}

\newtheorem*{remark}{Remark}

\begin{document}

\maketitle
\begin{abstract}
    During physical human robot collaboration, it is important to be able to implement a time-varying interactive behaviour while ensuring robust stability. Admittance control and passivity theory can be exploited for achieving these objectives. Nevertheless, when the admittance dynamics is time-varying, it can happen that, for ensuring a passive and stable behaviour, some spurious dissipative effects have to be introduced in the admittance dynamics. These effects are perceived by the user and degrade the collaborative performance.  In this paper we exploit the task redundancy of the manipulator in order to harvest energy in the null space and to avoid spurious dynamics on the admittance. The proposed architecture is validated by simulations and by experiments onto a collaborative robot.

\end{abstract}
\section{Introduction}

Kinematically redundant robots are endowed with more degrees of freedom (DOFs) than those required for the execution of a specific task. This extra dexterity of theirs can be exploited for implementing additional sub-tasks alongside the main one. This capability is pivotal in a Human Robot Collaboration (HRC) scenario, in which the applications often require high flexibility and dexterity.

In order to exploit the kinematic redundancy \cite{siciliano2010robotics}, it is possible to encode the tasks of the robot as functions which depend on the robot configuration, namely the task function \cite{samson1991robot}. We can then map the resulting control input from the task space into the joint space using the (pseudo-) inverse of the Jacobian, as shown in \cite{Khatib1987IJRA} and \cite{slotine1991general}. This method has been extensively used for the likes of redundant manipulators \cite{Chiaverini1997TRO}, humanoids \cite{Mansard2009TAC} and high-dimensional robots \cite{Flacco2013RSJ}.

Furthermore, it is possible to consider redundancy with respect to a specific task. Exploiting the concept of task redundancy, it is possible 
to simultaneously implement secondary tasks which do not affect the execution of the main one. This is often carried out by exploiting the null space projection of the Jacobian, as in \cite{Siciliano1991book,ott2015prioritized,escande2014hierarchical}.

The concept of redundancy has been also exploited in the control of interaction between the robot and the environment. In particular, control strategies for implementing impedance control on redundant robots have been proposed in \cite{ott2008cartesian}, \cite{albu2003cartesian} and \cite{lippiello2012exploiting}.

Admittance control \cite{Villani2016} is used for controlling the interaction in velocity controlled robots. The main idea is to control the robot to mimic a desired physical dynamics, i.e. the admittance dynamics. Since physical dynamics are passive,  a robustly stable interaction with the environment is guaranteed, namely the system is stable both in free motion and in interaction phase (see e.g. \cite{Secchi2007}).

While operating using admittance control, the interactive behaviour is entirely determined by the choice of the admittance parameters, such as inertia and stiffness. A specific set of parameters might not be optimal for the whole duration of the task \cite{Dimeas2016ToH}.

However, varying online the dynamic parameters can lead to the loss of passivity \cite{Ferraguti2015TRO} and thus to the implementation of unstable behaviours for HRC \cite{CTLandi2019IJRR}.

This limitation onto the flexibility of admittance control is due to the embodiment of passivity into the admittance dynamics. Enforcing the passivity of the specific motion poses much tighter restrictions compared to simply requiring the passivity of the \fb{overall} controlled system. Passivity itself, in fact, does not depend on any given interaction dynamics, only onto the energy flow in the system.

Following this line, the Time Domain Passivity Approach (TDPA) keeps track of the energy flow through a power port of the system and exploits a variable damper to compensate for possible  energy production \cite{ryu2004stability,kim2007landing}. However, the passivating damper acts as a spurious dynamics that affects the desired dynamic behavior. The problem of distributing the energy dissipated over the different subspaces is tackled in \cite{ott2011subspace}. Nevertheless, due to the lack of an exciting input in the null space, the amount of energy that can be dissipated is limited and it can be  necessary to resort to a damping in the main task. 

A different strategy was adopted in \cite{Ferraguti2015TRO}, \cite{CTLandi2019IJRR} and \cite{benzi22tro}, in which a variable admittance controller was passively implemented by leveraging energy tanks. Here, energy tanks \cite{Ferraguti2015TRO} have been exploited in order to unbind passivity from the specific admittance dynamics. This allows to maximise the flexibility of the application while maintaining a passive energy balance. 
Nevertheless, implementing a time-varying admittance dynamics can deplete the tank. Thus, it can happen that no more energy is available for implementing the desired behavior. 
Different solutions have been proposed in literature for facing this issue. Authors of \cite{Ferraguti2015TRO} and \cite{riggio2018use} design a specific additional damping term for refilling the tank whenever passivity is threatened, \cite{CTLandi2019IJRR} saturates the amount of energy which can be extracted from the tank at a given time, while \cite{benzi22tro} passively approximates the desired behaviour whenever the current energy level is too low for implementing it. All of these solutions involve an approximation of the desired dynamics, which can compromise the effectiveness and efficiency of the application. 
In order to achieve a variable admittance free of spurious limiting effects, it is important to separate the tank refill procedure from the desired interaction dynamics.

In this paper, we propose a control architecture which exploits task redundancy and energy tanks in order to faithfully reproduce a desired behaviour while preserving a passive energy balance. By means of kinematic decomposition and null space projection, we generate motions which are kinematically decoupled from the main admittance task, and use them to store energy in the tank for controlling the energy flow in a non-invasive way in terms of task execution. This allows to implement a variable admittance controller in a passive way without altering the implementation of the desired behaviour.

Even if energy tanks have been previously exploited in a null-space scenario in \cite{dietrich2015passivation}, they were utilised for passifying the dynamics implemented in a hierarchy of null spaces; thus, the main goal was using the tank for guaranteeing the passivity of the null space.

In this paper, we make use of the dynamics in the null space for harvesting energy and refilling the tank. This novel approach exploits the null space of the task for keeping the tank at an energy level sufficient for ensuring both performances and the overall passivity of the system.

Thus, the contributions of this paper are:
\begin{itemize}
    \item A non-invasive control architecture capable of passively implementing a variable admittance controller
    \item A technique for exploiting the task redundancy of the robot in order to refill the energy tank without compromising the execution of the main task
    \item An experimental validation of the proposed architecture, both in simulation and onto a real collaborative robot
\end{itemize}
The paper is organised as follows: in Sec. \ref{sec:problem} the main problem addressed in this paper is formulated. In Sec. \ref{sec:tanks} energy tanks are presented, alongside their role in the implementation of a variable admittance controller. In Sec. \ref{sec:null_refill} the exploitation of task redundancy for storing energy in the tank is presented. Finally, in Sec. \ref{sec:experiments} the resulting architecture is experimentally validated onto a UR10e manipulator, while  in Sec. \ref{sec:conclusions} conclusions are drawn and future works are discussed.

\section{Problem Statement}
\label{sec:problem}

Consider a fully actuated and velocity controlled $n$-DOFs manipulator that can be modelled by the following kinematic model: 
\begin{equation}\label{eq: robot_model}
    \dot{x} = J(q)u
\end{equation}
\fb{where $q\in\mathbb{R}^n$ is the vector of joint variables, $u \in \mathbb{R}^{n}$ is the control input, $x\in\mathbb{R}^m$ the pose of the end-effector and $\dot{x}\in\mathbb{R}^m$ its velocity, both expressed in operational space, while $J(q) \in \mathbb{R}^{m \times n}$ is the Jacobian matrix.}

We consider a human and a robot physically collaborating for executing a task. The physical interaction between the robot and the human is regulated by means of a variable admittance controller in order to make the collaboration safe and efficient \cite{CTLandi2019IJRR}. The admittance dynamics governs only a given subset of the workspace of dimension $m_1 < m$, meaning that the interaction \fb{task takes place only onto certain operational directions}. For tasks of simple nature, this is most often sufficient. 

The interaction task is given by $x_1=f_1(q)$, where ${x_1\in\mathbb{R}^{m_1}}$ and $f_1$ is the forward kinematic mapping that relates the task variable with the joint variables. By differentiating, it follows that
\begin{equation} \label{eq: diff_task_1}
     \dot{x}_1 = J_1(q)\dot{q}\,,\,J_1(q)= \frac{\partial f_1}{\partial q}
 \end{equation}
\fb{where $J_1 \in \mathbb{R}^{m_1 \times n}$ is the analytic task Jacobian, $\dot{q} \in\mathbb{R}^n$ are the joint velocities} and $\dot x_1$ are the \fb{operational task velocities} along the considered directions.
We now assume to have a certain degree of kinematic redundancy w.r.t. the task, i.e. task redundancy, meaning that $m_1 < n$. Then, the control input $u = \dot{q}$ implementing the task in \eqref{eq: diff_task_1} can be synthesized as
\fb{
\begin{equation}\label{eq: weighted_pinv_}
    \dot{q} = J_1^{+}(q)\dot{x}_1 + (I - J_1^{+}(q)J_1(q))z
\end{equation}
where $J_1^{+} \in \mathbb{R}^{n \times m_1}$ is the Moore-Penrose pseudo-inverse \cite{doty1993theory} of $J_1$, while $z \in \mathbb{R}^{n}$ are joint velocities which get projected in the null space of the task.}

 \begin{figure}[tb]\label{fig: adm_scheme}
     \centering
     \includegraphics[width=0.9\columnwidth]{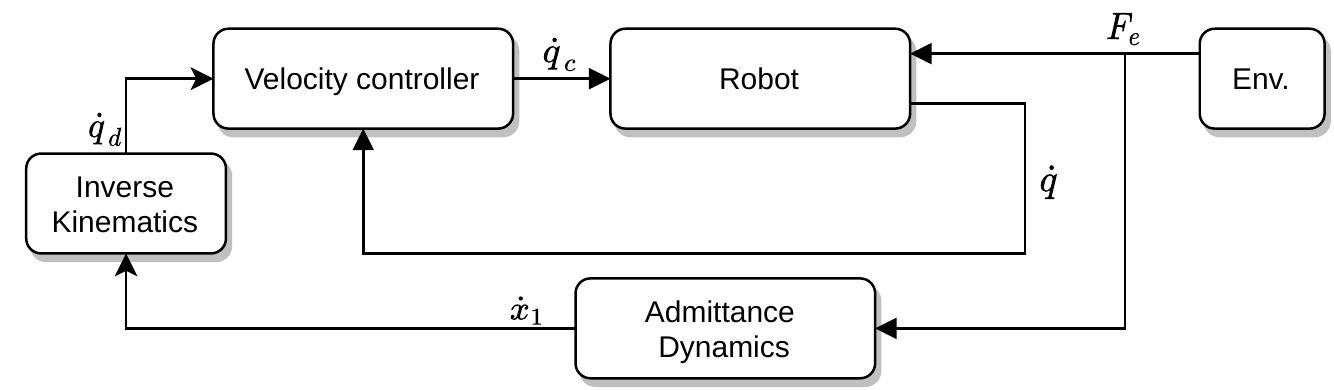}
     \caption{The admittance control architecture. We define as $\dot{q}_d$ the commanded input to the velocity controller, while $\dot{q}_c$ represent the velocity set point imposed to each joint by the robot actuation. We assume that the low-level controllers of the robot are such that $\dot{q}_d = \dot{q}$.}
 \end{figure}
The admittance controller architecture is portrayed in Fig.~1 and the controlled dynamic behaviour is represented by:
\begin{equation}\label{eq: admittance model}
    M(x_1,t)\ddot{x}_1 +  D(x_1)\dot{x}_1  = G F_{e}
\end{equation}
 where $M(x_1,t),D(x_1,t)\in\mathbb{R}^{m_1\times m_1}$ are the symmetric positive definite inertia matrix and damping matrix. $F_e\in\mathbb{R}^m$ is the force the human is applying on the end-effector. $G\in\mathbb{R}^{m_1\times m}$ is a selection matrix that extracts the components of $F_e$ corresponding to the admittance control directions. The admittance dynamics \eqref{eq: admittance model} is the most commonly used in human robot collaboration. All the results in the paper can be easily extended to more complex dynamics.

In order to cope with the loss of passivity due to the variability of the admittance parameters, a tank based admittance control will be implemented.  Our goal is to exploit the redundancy of the robot with respect to the task for refilling the energy tank. Specifically, we aim at exploiting the null-space motions for dissipating the required energy, without altering the desired execution of the main task. This allows us to obtain a flexible and robust interactive behaviour without compromising the nominal execution of the task, thus enhancing the performance of the application.

\section{Tank-based admittance control}
\label{sec:tanks}
In this section we present how to exploit energy tanks \cite{Ferraguti2015TRO} for reproducing a variable admittance dynamics on the robot primary task following the approach presented in \cite{CTLandi2019IJRR}.

Energy tanks are virtual energy reservoirs, which store the energy dissipated by the system for future usage. They can generally be represented by
\begin{equation}
    \label{eq:tankeqs}
    \begin{cases}
    \dot{x}_t(t) = u_t(t) \\
    y_t(t) = \frac{\partial T}{\partial x_t} = x_t(t)
    \end{cases}
\end{equation}
in which $x_{t}(t) \in \mathbb{R}$ is the state of the tank, $(u_t(t), y_t(t)) \in \mathbb{R} \times \mathbb{R}$ is the power port handling the energy flowing through the tank, while 
\begin{equation} \label{eq: tank_energy}
    T(x_{t}(t)) = \frac{1}{2} x_{t}(t)^2    
\end{equation}
is the energy stored in the tank.

In particular, this energy can be reused for passively implementing any desired control action.

The non-passive admittance dynamics \eqref{eq: admittance model} can be passively implemented by augmenting it with an energy tank as follows:
\begin{equation}\label{eq: vadm_tank}
    \begin{cases}
        M(x_1,t)\ddot{x}_1 + D(x_1,t)\dot{x}_1 = F_{1} \\
        \dot{x}_{t}(t) = \frac{\phi(t)}{x_{t}(t)}P_D(t) - \frac{\gamma(t)}{x_{t}(t)}P_M(t)
    \end{cases}
\end{equation}
in which $F_1 = GF_e$, $F_1 \in \mathbb{R}^{m_1}$, while
\begin{equation}\label{eq: powers_def}
    P_D(t)=\dot{x}_1(t)^T D(t) \dot{x}_1(t) \quad P_M(t)= \tfrac{1}{2}\dot{x}_1(t)^T \dot{M}(t) \dot{x}_1(t) \footnote{The dependence on $x_1$ is omitted here for readability reasons}
\end{equation}
$P_D(t)\geq 0$ is the power dissipated by the admittance dynamics, that fills the tank. $P_M(t)$ is sign indefinite and it is the power required for the varying the inertia; it can either refill the tank (passive variation) or extract energy from the tank (non passive variation).

We define two additional terms $\phi(t)$ and $\gamma(t)$ to enable/disable the energy storing process. This serves to limit the energy intake by inserting an upper bound, namely $T(x_t(t)) \leq \overline{T}$. In fact, an excessive amount of stored energy in the tank may lead to the implementation of practically unstable behaviours \cite{lee2010passive}. This upper bound is application dependant and is ensured by the presence of $\phi(t) \in \{0,1\}$ and $\gamma(t) \in \{0,1\}$, whose values are defined as 
\begin{equation}\label{eq: phi_def}
    \phi = 
    \begin{cases}
    1 \quad  \text{if $T(x_t(t)) \leq \overline{T}$} \\
    0 \quad  \text{otherwise}
    \end{cases}
\end{equation}
which regulates the storage of dissipated energy, and 
\begin{equation}\label{eq: gamma_def}
    \gamma   = 
    \begin{cases}
    0 \quad \text{if $\dot{M}(t) \leq 0$} \\
    1 \quad \text{otherwise}
    \end{cases}
\end{equation}
which always allows energy to be extracted from the tank when $\dot{M}(t) > 0$, while blocking the injection if $\dot{M}(t) \leq 0$.

In order to avoid singularities in \eqref{eq: vadm_tank}, it is necessary that $x_t\ne 0$, i.e. that some energy is always present in the tank. Formally, it is necessary that $T(x_t(t)) \geq \underline{\varepsilon}>0$ for all $t\geq 0$.
Using \eqref{eq: tank_energy} with \eqref{eq:tankeqs} and \eqref{eq: vadm_tank}, it is easy to show that
\begin{equation}\label{eq: tank_energy_flow}
    \dot{T}(x_t(t)) = x_t(t) \dot{x}_t(t) = \phi(t)P_D(t) - \gamma (t)P_M(t)
\end{equation}
Using the augmented dynamics in \eqref{eq: vadm_tank} together with \eqref{eq: tank_energy_flow} it is possible to prove 
\newtheorem{prop}{Proposition}
 \begin{prop} \label{prop: always passive}
 	If $T(x_{t})\geq \underline{\varepsilon}$ for all $t\geq0$, then the system \eqref{eq: vadm_tank} remains passive with respect to the pair $(F_1(t), \dot{x}_1(t))$.
 \end{prop}
The proof can be found in \cite{CTLandi2019IJRR}.

Thus, as long as a sufficient amount of energy is stored in the tank, it is possible to reproduce any desired behaviour, even non-physical ones, in a passive way. This allows us to vary online the dynamic parameters of the admittance controller while preserving a robustly stable interaction with the environment.

Nevertheless, the implementation of non passive behaviors (e.g. inertia variation) depletes the tank. Thus, it can happen that there is not enough energy in the tank for passively implementing the desired behavior. In this case, see e.g.\cite{Ferraguti2015TRO}, the damping of the admittance dynamics is augmented for refilling the tank. Such an action is effective but the price to pay is \fb{an unwanted dynamic behavior (additional damping) that is perceived by the user, undermining the quality of the interaction.} This problem will be solved by moving the damping injection in the null space, as shown in Sec.~\ref{sec:null_refill}.

\section{Null Space Refill}
\label{sec:null_refill}
In this section we will present the proposed technique for exploiting task redundancy in order to refill the energy tank without altering the nominal behaviour of the robot.

First, we need to ensure that secondary motion does not affect the execution of the main task, i.e. the two motions must be kinematically decoupled. To ensure this, we perform the following coordinate transformation\cite{park1999dynamical,michel2020passivity} onto \eqref{eq: weighted_pinv_}: 
\begin{align}\label{eq: kin_decoupl}
    \dot{q} &= \begin{bmatrix}
                    J_1(q)^{+} \, Z(q)^{T}
                \end{bmatrix}
                \begin{bmatrix}
                    \dot{x}_1\\
                    v_2
                \end{bmatrix}
\end{align} 
\fb{in which $Z(q) \in \mathbb{R}^{m_2 \times n}$ is a full row rank null space base matrix of $J_1(q)$ and can be easily obtained via singular value decomposition (SVD) of $J_1(q)$. }Finally, the self motion velocity $v_2 \in \mathbb{R}^{m_2}$ is computed as
\begin{equation}\label{eq: self_motion}
    v_2 = N(q)\dot{q}
\end{equation}
in which $N(q)= (Z(q)Z(q)^T)^{-1}Z(q)$. We can now go ahead and define the extended Jacobian of the system as
\begin{align}\label{eq: extended_jacobians}
    J_e(q) &= \begin{bmatrix}
                J_1(q) \\
                N(q)
              \end{bmatrix}
              \, , \,
    J_e(q)^{-1} = 
              \begin{bmatrix}
                J_1(q)^{+} \,\,
                Z(q)^T
              \end{bmatrix}
\end{align}
and the extended task-space velocities as
\begin{align}\label{eq: extended_velocities}
    \dot{x}_e &= \begin{bmatrix}
                    \dot{x}_1 \\
                    v_2
                \end{bmatrix}
                = J_e(q) \dot{q}
\end{align}
and finally compute the velocity commands in the joint space implementing the combined motions as
\begin{equation}\label{eq: joint_space_sol}
         \dot{q} = J_e(q)^{-1}\dot{x}_e
\end{equation}
In this way, we achieved the kinematic decoupling between the two velocities $\dot{x}_1$ and $v_2$, namely $J_1(q) Z(q)^T v_2 = 0$ and $N(q) J_1(q) ^{+}\dot{x}_1 =0$. A more in-depth explanation of this procedure can be found in \cite{michel2020passivity}.

Following this procedure, we can then implement a secondary task whose execution does not interfere with the primary one, exploiting the task redundancy of the robot.
Since our goal is the implementation of a time-varying dynamics as the main admittance task, we can design the secondary task as a supplement to the main one, aimed at injecting energy in the tank whenever a refill is required.

We achieve this by introducing an extra dynamics onto the secondary motion $v_2$. In particular, we can endow this dynamics with a damper and utilise the energy dissipated by the secondary motion for refilling the tank. In this way, there is no need to introduce spurious terms onto the main dynamics, which can be faithfully reproduced.

We can perform this by further augmenting the dynamics in \eqref{eq: vadm_tank} as:
\begin{equation}\label{eq: vadm_tank_null}
    \begin{cases}
        M(x_1,t)\ddot{x}_1 + (D(x_1,t) + \Psi(t))\dot{x}_1 = F_{1} \\
        \dot{x}_{t}(t) = \frac{\phi(t)}{x_{t}(t)}(P_D(t) + P_N(t)) - \frac{\gamma(t)}{x_{t}(t)}P_M(t) \\
        \dot{v}_2 = -D_N(t)v_2 + F_{null}
    \end{cases}
\end{equation}
in which $D_N(t) \in \mathbb{R}^{m_2 \times m_2}$, $D_N(t) \geq 0$ is the damping term for the null space motion, $F_{null} \in \mathbb{R}^{m_2}$ is a control force acting onto the null space coordinates, while
\begin{equation}\label{eq: p_null}
    P_{N}(t) = v_2(t)^T D_N(t) v_2(t)
\end{equation}
is the energy dissipated by the damped null motion. 
Through this augmentation, we obtain that
\begin{equation}\label{eq: tank_power_augm}
    \dot{T}(x_t(t)) = \phi(t)(P_D(t) + P_N(t)) - \gamma (t)P_M(t)
\end{equation}
which shows that the energy dissipated by the secondary dynamics can be utilised for refilling the tank and for implementing a non passive inertia variation. In this way, we can passively implement a variable admittance controller without introducing spurious elements in the desired admittance dynamics.

The force $F_{null}(t)\in \mathbb{R}^{m_2}$ is activated when the energy of the tank drops below a threshold $T^*<\bar T $ and is defined as
\begin{equation}\label{eq: null_force_def}
    F_{null_{i}}(t) = 
    \begin{cases}
    h_i(t) &\quad \text{if $T((t)) < T^*$} \\
    0 &\quad \text{otherwise}
    \end{cases} \quad i=\{1,\dots, m_2\}
\end{equation}
in which $h_i(t)$ is appropriately designed for exciting the null dynamics in order to produce motion and dissipate energy.

The time-varying positive definite damping matrix implemented in the null dynamics is defined as $D_N(t)=d_N(t)I_{m_2}$, where 
\begin{equation}\label{eq: null_damp_def}
    d_{N}(t) = 
    \begin{cases}
        max\left(\frac{\gamma P_M -  P_D}{v_2^T v_2}, \delta \right) \quad &\text{if}\,\, T(t)=\underline{\varepsilon} \,\,\text{and}\,\, v_2 \neq 0\\
        \,\,\delta &\text{otherwise}
    \end{cases}
\end{equation}
where $\delta >0$ is the minimum damping that is always implemented in the null dynamics. \fb{The presence of a non zero damping makes the null space dynamics output strictly passive and, therefore, stable \cite{van2000l2}}.

Using the null space damping dynamics in \eqref{eq: vadm_tank_null} and in \eqref{eq: null_damp_def} it is possible to  ensure that the tank is never depleted if  $v_2 \neq 0$. 
In fact, as long as $T(t)>T^*$, $F_{null}=0$ and the null space energy harvesting is not enabled since the energy dissipated by the main task is enough for refilling the tank. When $T(t)\leq T^*$, the null space dynamics is activated and an extra dissipated power $P_N(t)=v_2^Td_N(t)v_2=v_2^T\delta v_2$ is fed into the tank for further refilling it,  without disturbing the main task. When the tank reaches its minimum value $\underline{\varepsilon}>0$, using \eqref{eq: null_damp_def}, the damping of the null dynamics is tuned in such a way to have $P_N(t)=v_2^Td_N(t)v_2\geq \gamma P_M - P_D$. Then, from \eqref{eq: tank_power_augm} we have that $\dot{T}\geq 0$ and, therefore, an increase of the energy in the tank. Thus, the tank can never go below the minimum energy threshold, i.e. the tank can never deplete. 

By properly choosing the damping parameter $\delta>0$, it is possible to keep the energy of the tank above $\underline{\varepsilon}$ and, therefore, to avoid the velocity dependent damping reported in \eqref{eq: null_damp_def}. Nevertheless, it can happen that the minimum is reached and, if $v_2=0$ the necessary control action cannot be implemented. Energy then needs to be recovered by adding an extra damping to the main task. This is done through the term $\Psi=\psi I_{m_1}$ where
\begin{equation}\label{eq: damp_inj_main}
    \psi (t)=
    \begin{cases}
        max(\frac{P_M-P_D}{\dot{x}_1^T\dot{x}_1},0)     &\text{if} \,\,T (t)= \underline{\varepsilon} \,\, \text{and}\, v_2 = 0, \, \dot{x}_1 \neq 0\\
        \,\,0 &\text{otherwise}
    \end{cases}
\end{equation}
Notice that, if $\dot{x}_1 = 0$, $P_M = 0$, thus no energy is extracted from the tank.
It is easy to see that this damping term plays the same role as the null space damping term when $T(t)=\underline{\varepsilon}$ with the main difference that it produces an unwanted effect on the main task. Nevertheless, if the null space dynamics is properly designed, $\Psi$ is almost never active.



Combining \eqref{eq: vadm_tank_null} with \eqref{eq: tank_power_augm} we can finally prove the following.
\newtheorem{prop2}{Proposition2}
 \begin{prop} \label{prop: always passive_null}
 	If $T(x_{t})\geq \underline{\varepsilon}$ for all $t\geq0$, then the system \eqref{eq: vadm_tank_null} remains passive with respect to the pairs $(F_1(t), \dot{x}_1(t))$ and $(F_{null}(t), v_2(t))$.
 \end{prop}
\begin{proof}
Let us consider the following positive storage function
\begin{equation} \label{eq: storage_function}
    S(\dot{x}_1(t), x_t(t), v_2(t)) = H_1(\dot{x}_1(t)) + T(x_t(t)) + H_2(v_2(t))
\end{equation}
in which $H_1(\dot{x}_1(t))$ is defined as 
\begin{equation} \label{eq: H1_function}
    H_1(\dot{x}_1(t)) = \frac{1}{2}\dot{x}_1^T(t) M(t) \dot{x}_1(t)
\end{equation}
while $H_2(v_2(t))$ as
\begin{equation} \label{eq: H2_function}
    H_2(v_2(t)) = \frac{1}{2}v_2^T(t) v_2(t)
\end{equation}
and $T(x_t(t))$ as in \eqref{eq: tank_energy}. From now onward, we indicate as $S(t)$, $H_1(t)$, $H_2(t)$ and $T(t)$ the values of $S(\dot{x}(t), x_t(t))$, $H_1(\dot{x}_1(t))$, $H_2(v_2(t))$ and $T(x_t(t))$ at time $t$ for the sake of readability. By differentiating, we obtain:
\begin{multline} \label{eq: dot_s}
    \dot{S}(t) = \dot{H}_1(t) + \dot{T}(t) + \dot{H}_2(t)= \dot{x}_1^T F_1 - (1 - \phi(t))P_D(t)+ \\-(1 - \phi(t)) P_N(t) + (1 - \gamma(t))P_M(t) + v_2^T F_{null} 
\end{multline}
from which, since $\phi(t) \in \{ 0,1\}$, $P_D(t) \geq 0$, $P_N(t) \geq 0$, we have that
\begin{multline}
    \dot{x}_1^T F_1 + v_2^T F_{null}\geq \dot{H}_1(t) + \dot{T}(t) +\dot{H}_2(t) - (1 - \gamma(t))P_M(t)
\end{multline}
From \eqref{eq: powers_def} and \eqref{eq: gamma_def}, we get that if $\dot{M}(t) \leq 0$ then $P_M = 0$. Otherwise, if $\dot{M}(t) > 0$ then $(1- \gamma(t))P_M = 0$. Thus, we obtain that
\begin{equation}\label{eq: power_inequality}
    \dot{x}_1^T F_1 + v_2^T F_{null}\geq \dot{H}_1(t) + \dot{T}(t) +\dot{H}_2(t)
\end{equation}
which implies
\begin{multline}
    \int_{0}^{t}(\dot{x}_1^{T}(\tau)F_1(\tau) + v_2(\tau)^{T}F_{null}(\tau))d\tau\geq \\H_1(t) - H_1(0) + T(t) - T(0) +H_2(t) - H_2(0)\\ \geq  -H_1(0) - T(0) - H_2(0)
\end{multline}
thus proving passivity.
\end{proof}

\begin{remark}
    Notice that, even if the two motions $\dot{x}_1$ and $v_2$ are kinematically decoupled, the overall passivity is not guaranteed a priori, since the two dynamics are energetically coupled via the energy tank.  
\end{remark}

Passivity makes the overall system robust to external disturbances \cite{Secchi2007}, that can  happen both in the main and in the null space dynamics.
Passivity can only be proven with respect to the two pairs. Nevertheless, from the kinematic decoupling in \eqref{eq: kin_decoupl} we have that
 \begin{equation}\label{eq: explicit_decoupling}
     \dot{q} = J_1(q)^+\dot{x}_1 + Z(q)^{T}v_2
 \end{equation}
 Taking \eqref{eq: diff_task_1} into account, we obtain that
 \begin{equation}\label{eq: explicit_decoupling_2}
     \dot{x}_1 = J_1(q)\dot{q} + J_1(q)Z(q)^{T}v_2
 \end{equation}
 in which, from the previous kinematic decoupling, $J_1(q)Z(q)^{T}v_2 = 0$.
 In this way, the contribution of the pair $(F_{null}(t), v_2(t))$ does not affect the implementation of $\dot{x}_1(t)$. Therefore, since the two dynamics are passive w.r.t. their power ports and are passively interconnected via the energy tank, the overall passivity is preserved.
 
Thus, we are able to preserve a passive energy flow without having to introduce additional damping terms onto the main motion, guaranteeing at the same time both passivity and efficiency.
It should be noticed that the energy harvesting dynamics is a simple mass-damper dynamics and, as well known (see e.g. \cite{Secchi2007}), it is passive with respect to the pair $(F_{null},v_2)$. Thus the secondary task dynamics is bounded input/bounded output stable and, when $F_{null}=0$ it, it is asymptotically stable \cite{Secchi2007}.


\section{Simulations and Experiments}
\label{sec:experiments}
The experimental validation of the architecture has been conducted onto a Universal Robot 10e manipulator, equipped with an on-board 6-axis force/torque (F/T) sensor. Both the robot and the sensor run at a frequency of $500$Hz.

A first validation experiment was performed in a simulated environment, while an equivalent version of it was implemented on the real manipulator.
\subsection{Simulations}
\begin{figure}[tb]
  \centering
    \subfloat[\centering]{\includegraphics[width=.49\columnwidth]{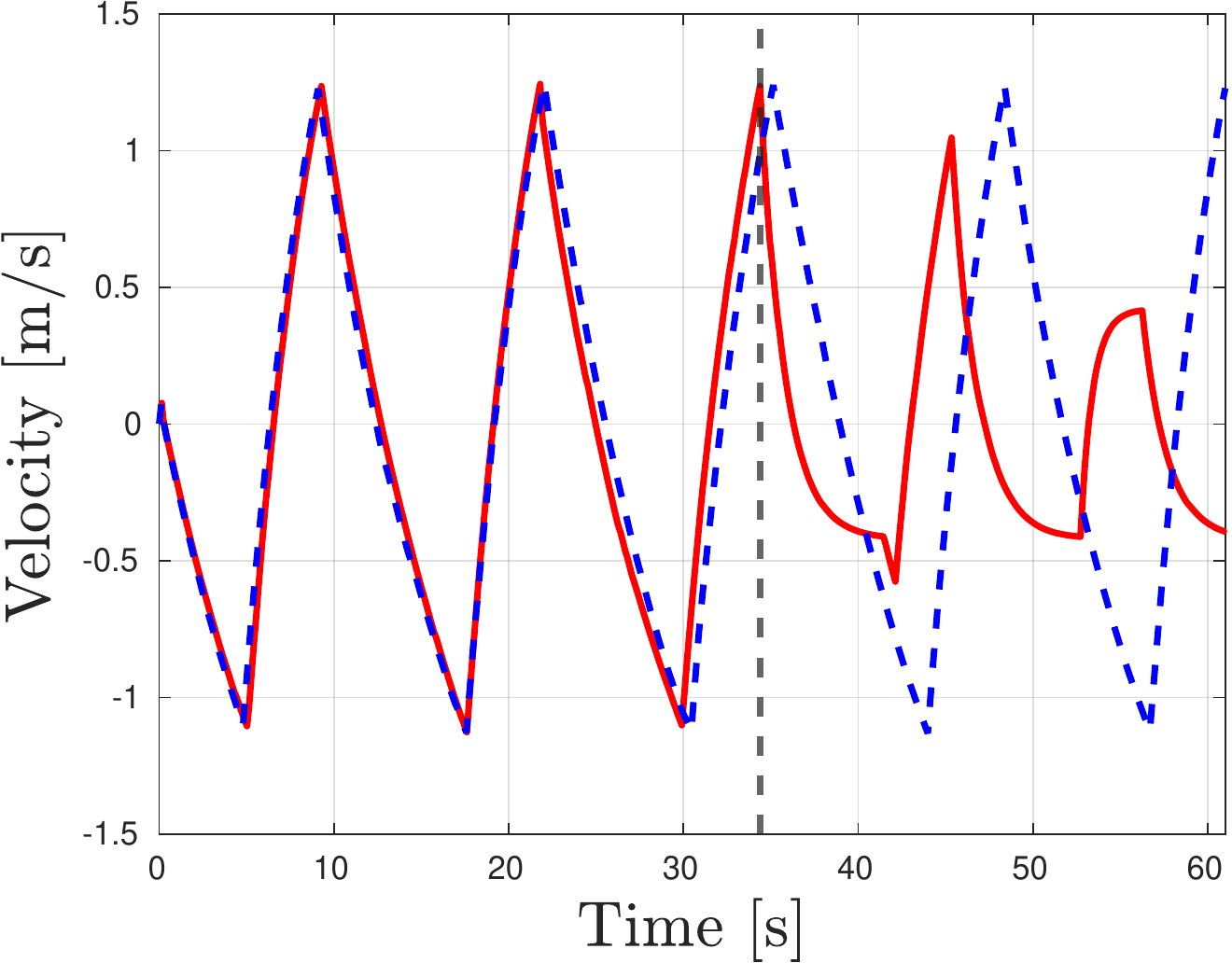}\label{fig: velocity_damp_inj} }%
    \subfloat[\centering]{\includegraphics[width=.49\columnwidth]{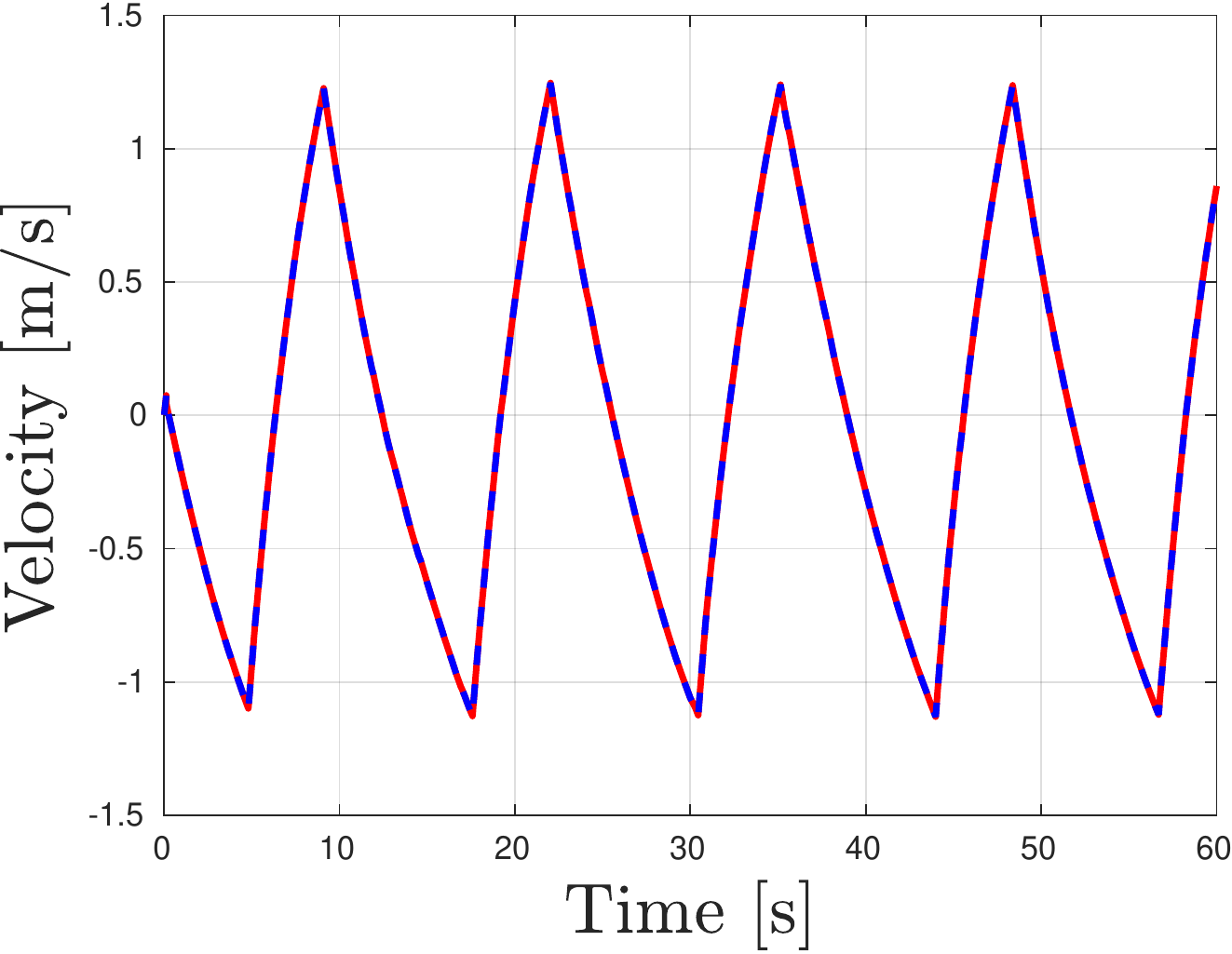} \label{fig: velocity_null_damp}}%
    \caption{Velocity of the end effector along the $z$ axis using  damping injection (a) and null-space refill (b) in \textit{red} and the desired motion in \textit{blue}. The dotted line at $t=34s$ in (a) represent the starting time of the damping injection.}%
    \label{fig: Velocity_damp_inj_vs_null}%
\end{figure}
\begin{figure}[tb]
  \centering
    \subfloat[\centering]{\includegraphics[width=.48\columnwidth]{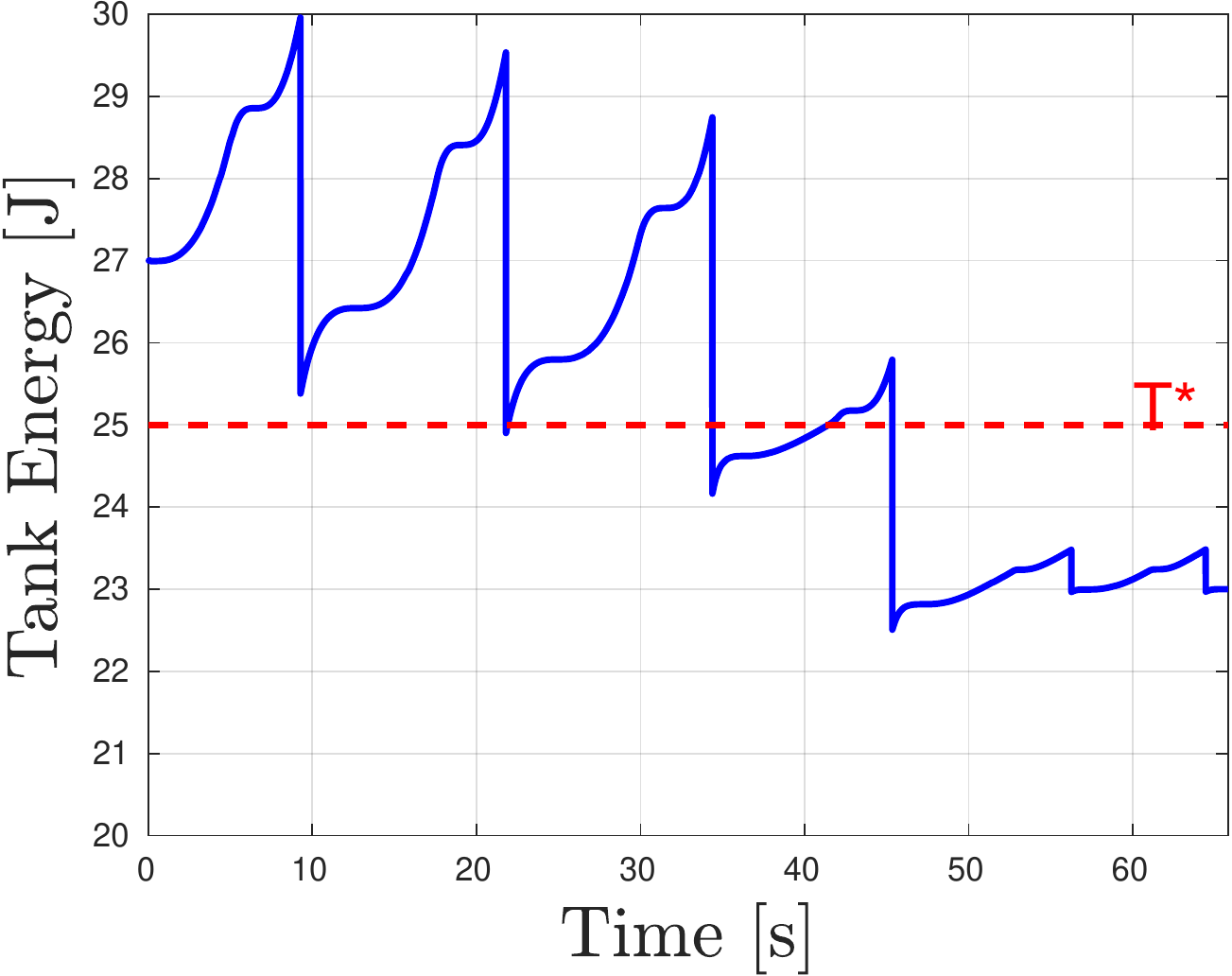}\label{fig: tank_damp_inj} }%
    \subfloat[\centering]{\includegraphics[width=.48\columnwidth]{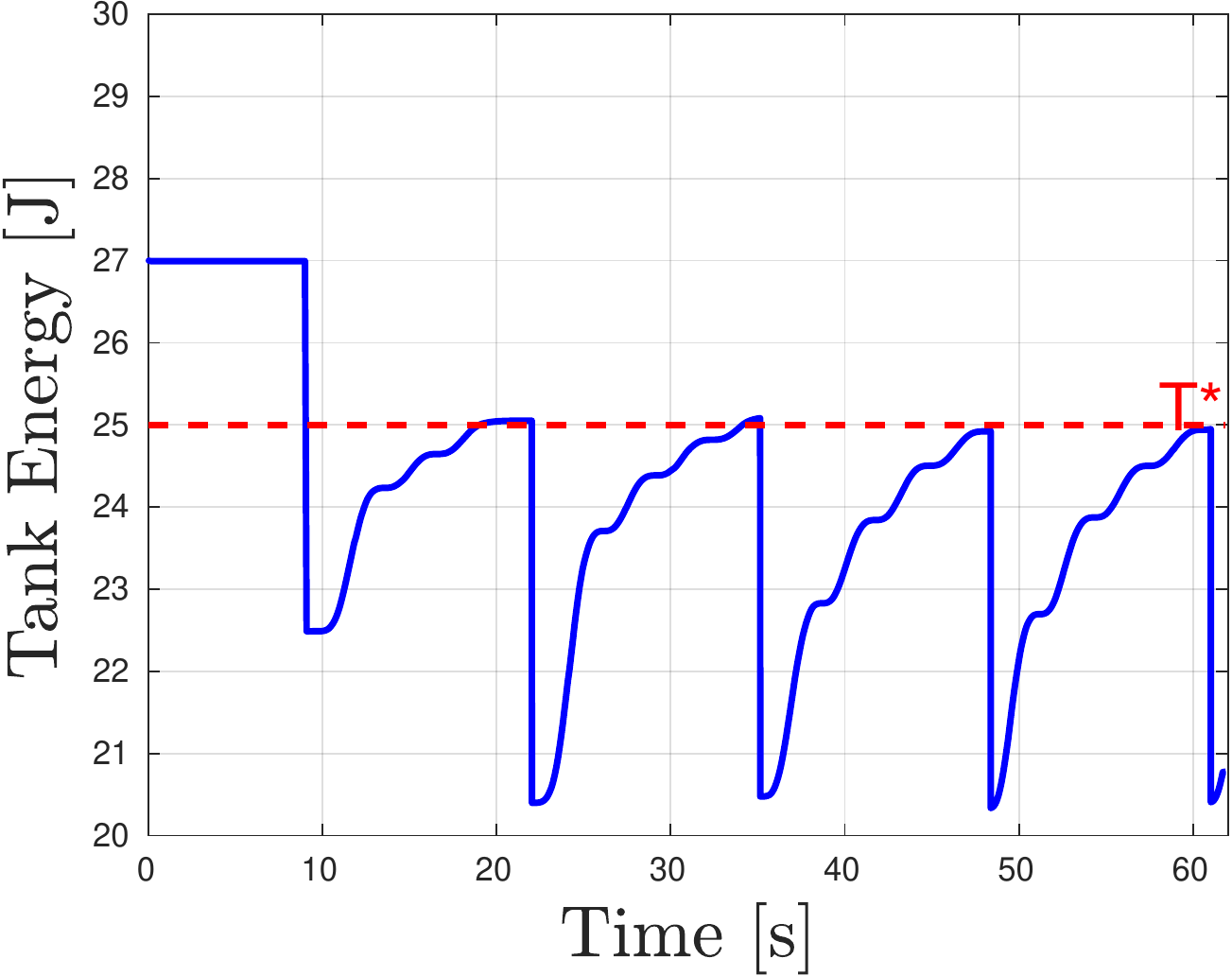} \label{fig: tank_null_damp}}%
    \caption{Evolution of the energy in the tank using damping injection (a) and null-space refill (b).}%
    \label{fig: Tank_damp_inj_vs_null}%
\end{figure}
In the simulation, the robot is employed to reproduce the behaviour in \eqref{eq: admittance model} while aiming at preserving a desired amount of energy $T^* = 25J$ in the tank. This is first accomplished using the standard tank-based admittance control in \eqref{eq: vadm_tank}, then with the proposed null-space refill technique in \eqref{eq: vadm_tank_null}.

We assume the \fb{admittance task} takes place only along the translational directions of motion. Thus, the primary task $x_1$ has dimension $m_1 = 3$ and its Jacobian $J_1$ is equivalent to the first $m_1$ rows of the geometric Jacobian $J$. Since $m = n = 6$, we can exploit the remaining $m_2 = 3$ DOFs for implementing the secondary task $x_2$.\\
Under these conditions, the selection matrix $G$ becomes
\begin{equation}\label{eq: matrix_G}
    G =
    \begin{bmatrix}
       \mathbb{I}_3\quad \mathbb{O}_3\\
    \end{bmatrix}
\end{equation}
where $\mathbb{I}_3$ is the 3D identity matrix and $\mathbb{O}_3$ is the 3D zero matrix.
The inertia and damping matrices $M(x_1, t)$ and $D(x_1, t)$ are initialised as diagonal matrices, whose elements have respectively the value of $6kg$ and $0.75 \frac{Ns}{m}$.

The external force $F_e$ presents only one non-null component $F_z$ with magnitude $2N$, applied along the vertical direction. The force cyclically inverts its sign. Whenever this happens, the value of $M(x_1,t)= diag(m_i(t))$ is changed, increasing or decreasing its elements by $3kg$ at each cycle.

As we can see from \eqref{eq: powers_def} and \eqref{eq: gamma_def}, increasing the inertia leads to an extraction of energy from the tank, while no energy is injected if the inertia decreases. Thus, since we aim at an energy level of $T^*$, a technique for refilling the tank is required.

First, the tank is refilled via the standard damping injection. Whenever the energy level $T(x_t(t))$ drops below $T^*$, the value of the elements of $D(x_1,t)$ is increased to $4 \frac{Ns}{m}$.

\begin{figure}[tb]
    \centering
    \includegraphics[width=0.75\columnwidth]{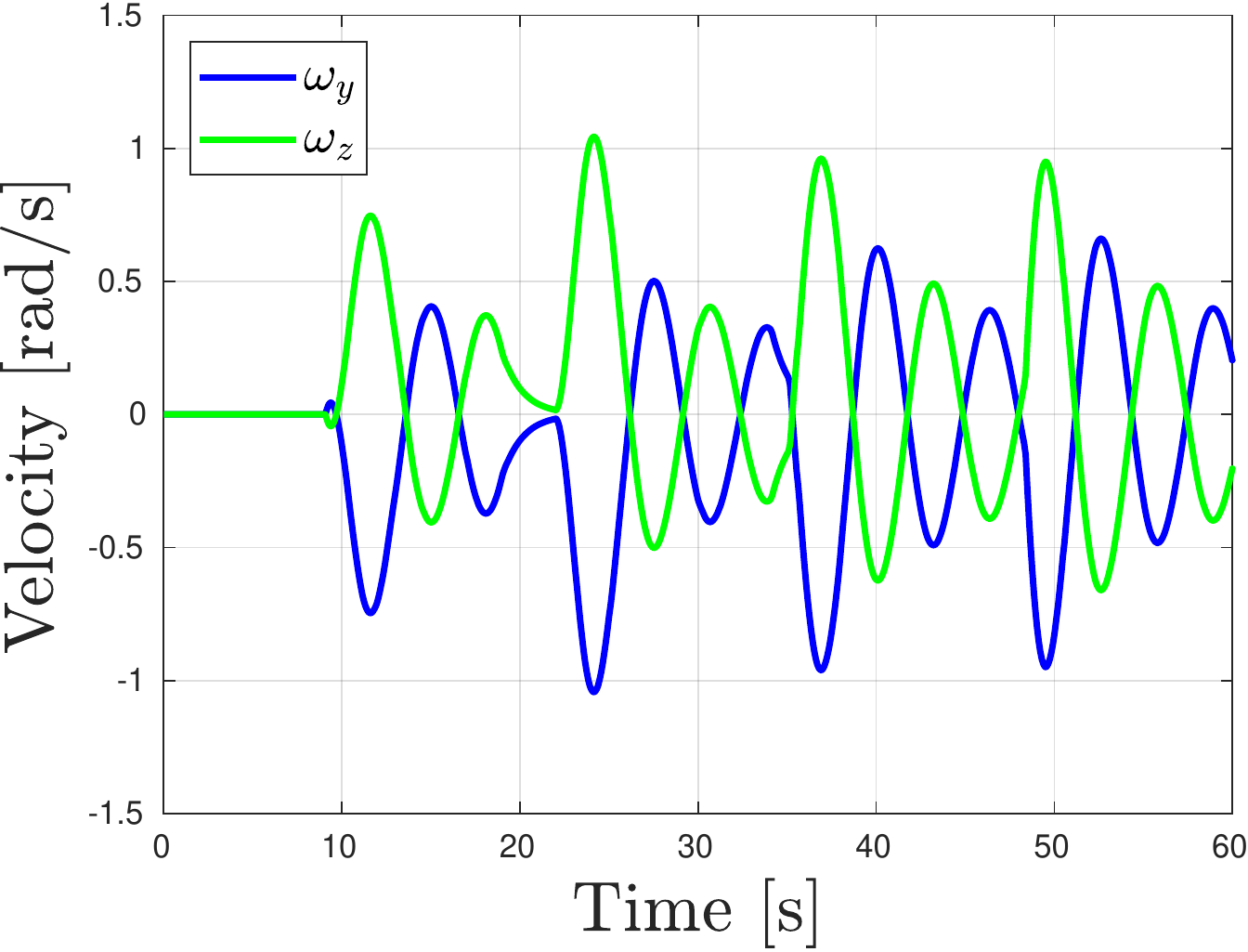}
    \caption{Velocities of the secondary motion using null space refill.}
    \label{fig: null_velocities}
\end{figure}


As can be noticed in Fig.~\ref{fig: velocity_damp_inj}, this solution severely affects the execution of the main task, \fb{resulting in a rough approximation of the desired motion after the damping injection}. Additionally, the energy level \fb{never rises back to $T^*$, due to the slower evolving nature of the dynamics, meaning that the additional damping would persist indefinitely} (see Fig.~\ref{fig: tank_damp_inj}).

The same simulation is then conducted using the proposed null-space refill technique. The null damping matrix $D_N(t) \in \mathbb{R}^{3 \times 3}$ is initialised as a diagonal matrix, with elements equal to $1 \frac{Ns}{m}$. The null force $F_{null} \in \mathbb{R}^3$ has components designed as periodic functions $h_i$, namely
\begin{equation}\label{eq: force_null}
    h_i = \gamma_i|T(x_t(t)) - T^*|(sin(\omega t))
\end{equation}
in which $\gamma_i \in \mathbb{R}$ is a proportional gain and $\omega = 1 \frac{rad}{s}$. The gains are chosen as $\gamma_1 = 0$, $\gamma_2 = 0.5$ and $\gamma_3 = -0.5$. Notice from \eqref{eq: null_force_def} that $F_{null_{i}} = 0$ if no refill is needed.

To highlight the capabilities of the technique, we set as the only refilling term in the tank the null power $P_N$, namely we impose $P_D = 0$ in \eqref{eq: vadm_tank_null}.

The results are showcased in Fig.~\ref{fig: velocity_null_damp} and Fig.~\ref{fig: tank_null_damp}. Whenever the energy level in the tank drops below $T^*$, the force $F_{null}$ is applied, allowing to damp onto the secondary motion $v_2$. \fb{Even without $P_D$, the tank is refilled up until the desired value $T^*$ after each extraction. After that, $h_i = 0\,N$, thus the null motion slowly stops, leading to $P_N = 0$}. The generated null-space motions are shown in Fig.~\ref{fig: null_velocities}.

From the kinematic decoupling of $x_1$ and $v_2$, we achieve that the main motion is unaffected by the secondary damped one, avoiding the introduction of spurious terms in the main dynamics. At the same time, it is always possible to inject energy in the tank, thus passivity is preserved at all times. 
\subsection{Experiments}
The experiment conducted onto the real robot is analogous to the second simulation. Here, the human operator controls the manipulator by means of the main admittance task, in order to paint lines on a paper sheet, using a brush attached to the end effector\fb{, while motions in the null space of the task serve the refill purpose. The operator is skilled and is aware of how the robot will behave during the null motions.}

The inertia, damping and null damping matrices are initialised with elements equal to $15kg$, $40 \frac{Ns}{m}$ and $7 \frac{Ns}{m}$ respectively. The gains $\gamma_{(.)}$ in \eqref{eq: force_null} are chosen as $\gamma_1 = \gamma_2 = 0$, $\gamma_3 = -0.5$, meaning the motion $v_2$ is restricted to the rotation of the end effector along its $z$ axis. \fb{Due to the choice of \eqref{eq: force_null}, the null motion only causes small deviations from the initial angular position of the end effector. Thus, the operator is minimally affected during the task execution.} As in the simulations, the inertia matrix is changed according to the sign of the force $F_z$, with variations of $\pm5kg$.

\begin{figure}[tb]
    \centering
    \includegraphics[width=0.75\columnwidth]{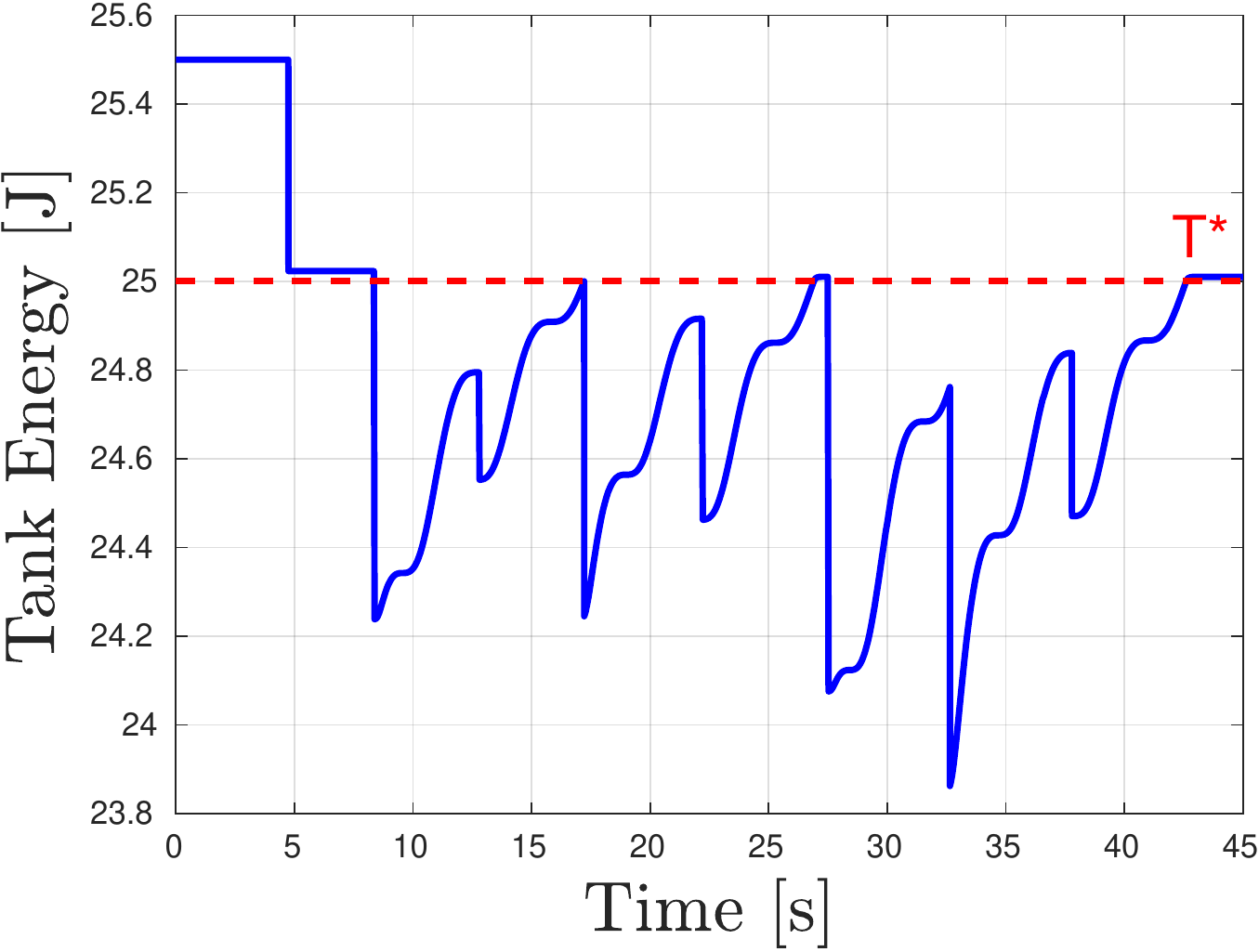}
    \caption{Evolution of the energy in the tank during the experiment}
    \label{fig: tank_exp}
\end{figure}
As previously, energy is extracted from the tank for compensating the inertia increase due to the change in orientation of the applied force. The activation of the null space dynamics allows to keep on implementing the desired dynamics. Thanks to the non-invasive nature of the motion, the operator is unaffected by the energy injection process, while the null-space refill refill ensures the preservation of passivity. The evolution of the energy in the tank during the experiment is portrayed in Fig.~\ref{fig: tank_exp}.
\section{Conclusions}
\label{sec:conclusions}
In this paper we have proposed a control architecture which exploits both task redundancy and energy tanks for passively implement a variable admittance controller. Through kinematic decoupling, we can generate damped motions in the null space of the main admittance task for refilling the energy tank whenever needed, thus avoiding the introduction of spurious effects onto the desired behaviour.

Given the projection-based nature of this approach, inequality and equality constraints, such as joint position limits or maximum permitted values of velocity and acceleration, cannot be incorporated. Future work aims at extending the framework proposed in the paper for including robot and task based constraints.


\bibliographystyle{IEEEtran}
\bibliography{IEEEabrv,mybib.bib}

\end{document}